\documentclass{article}
\usepackage[accepted]{icml2016}

 \usepackage{times,graphicx,subfigure,amsthm,comment,natbib,algorithm,algorithmic,amssymb,amsmath,amsfonts, mathtools, hhline, url, enumerate} 

\specialcomment{notes}{\begingroup\sffamily\color{blue}}{\endgroup}
\excludecomment{notes} 

\mathtoolsset{
showonlyrefs=true
}

\DeclareMathOperator*{\diag}{diag}

\DeclareMathOperator*{\expec}{\mathbb E}

\newcommand{\X}{{\mathbf X}}
\newcommand{\x}{{\mathbf x}}
\newcommand{\bs}{{\mathbf s}}

\newcommand{\bo}{{\mathbf o}}

\newcommand{\bC}{{\mathbf C}}
\newcommand{\bz}{{\mathbf z}}
\newcommand{\bd}{{\mathbf d}}
\newcommand{\ft}{{\mathbf f}}
\newcommand{\Ft}{{\mathbf F}}

\newcommand{\bt}{{\mathbf t}}

\newcommand{\Pt}{{\mathbf P}}
\newcommand{\Qt}{{\mathbf Q}}
\newcommand{\bbb}{{\mathbf b}}

\newcommand{\bh}{{\mathbf h}}

\newcommand{\y}{{\mathbf y}}

\newcommand{\bP}{{\mathbb{P}}}
\newcommand{\be}{{\mathbf{e}}}
\newcommand{\bp}{{\mathbf{p}}}

\newcommand{\cT}{{\mathcal{T}}}

\newcommand{\proj}{{\mathtt{Proj}}}

\newcommand{\wt}{\mathbf{w}}
\newcommand{\bc}{\mathbf{c}}
\newcommand{\Wt}{\mathbf{W}}

\newcommand{\vt}{\mathbf{v}}
\newcommand{\Vt}{{\mathbf V}}
\newcommand{\bO}{{\mathbf 1}}
\newcommand{\bR}{{\mathbb R}}

\newcommand{\dd}{{\partial}}

\newcommand{\bA}{{\mathbf A}}

\newcommand{\bD}{{\mathbf D}}
\newcommand{\bU}{{\mathbf U}}

\newtheorem{thm}{Theorem}

\newtheorem{cor}[thm]{Corollary}

\newtheorem{defn}{Definition}

\theoremstyle{remark}

\newenvironment{proofsketch}{\noindent{\emph{Proof sketch.}}}%
        {\hspace*{\fill}$\Box$\par}

\icmltitlerunning{Strongly-Typed Recurrent Neural Networks}

\begin{document} 
\twocolumn[
\icmltitle{Strongly-Typed Recurrent Neural Networks}

\icmlauthor{David Balduzzi$^1$}{dbalduzzi@gmail.com}
\icmlauthor{Muhammad Ghifary$^{1,2}$}{mghifary@gmail.com}
\icmladdress{${}^1$Victoria University of Wellington, New Zealand\\
${}^2$Weta Digital, New Zealand}

\icmlkeywords{recurrent neural networks, deep learning}

\vskip 0.3in
]

\begin{abstract} 
	Recurrent neural networks are increasing popular models for sequential learning. Unfortunately, although the most effective RNN architectures are perhaps excessively complicated, extensive searches have not found simpler alternatives. This paper imports ideas from physics and functional programming into RNN design to provide guiding principles. From physics, we introduce type constraints, analogous to the constraints that forbids adding meters to seconds. From functional programming, we require that strongly-typed architectures factorize into stateless learnware and state-dependent firmware, reducing the impact of side-effects. The features learned by strongly-typed nets have a simple semantic interpretation via dynamic average-pooling on one-dimensional convolutions. We also show that strongly-typed gradients are better behaved than in classical architectures, and characterize the representational power of strongly-typed nets. Finally, experiments show that, despite being more constrained, strongly-typed architectures achieve lower training and comparable generalization error to classical architectures.
\end{abstract}

\section{Introduction}
\label{sec:intro}

Recurrent neural networks (RNNs) are models that learn nonlinear relationships between sequences of inputs and outputs. Applications include speech recognition \citep{graves:13}, image generation \citep{gregor:15}, machine translation \citep{sutskever:14} and image captioning \citep{vinyals:15,karpathy:15}. Training RNNs is difficult due to exploding and vanishing gradients \citep{hochreiter:91,bengio:94,pascanu:13}. Researchers have therefore developed gradient-stabilizing architectures such as Long Short-Term Memories or LSTMs \citep{hochreiter:97} and Gated Recurrent Units or GRUs \citep{cho:14}.

Unfortunately, LSTMs and GRUs are complicated and contain many components whose roles are not well understood. Extensive searches \citep{bayer:09,jozefowicz:15,greff:15} have not yielded significant improvements. This paper takes a fresh approach inspired by dimensional analysis and functional programming.

\paragraph{Intuition from dimensional analysis.}
Nodes in neural networks are devices that, by computing dot products, measure the similarity of their inputs to representations encoded in weight matrices. Ideally, the representation learned by a net should ``carve nature at its joints''. An exemplar is the \emph{system of measurement} that has been carved out of nature by physicists. It prescribes units for expressing the readouts of standardized measuring devices (e.g. kelvin for thermometers and seconds for clocks) and rules for combining them. 

A fundamental rule is the \emph{principle of dimensional homogeneity}: it is only meaningful to add quantities expressed in the same units \citep{bridgman:22,hart:95}. For example adding seconds to volts is inadmissible. 
In this paper, we propose to take the measurements performed by neural networks as seriously as physicists take their measurements, and apply the principle of dimensional homogeneity to the representations learned by neural nets, see section~\ref{sec:types}.

\paragraph{Intuition from functional programming.}
Whereas feedforward nets learn to approximate functions, recurrent nets learn to approximate programs -- suggesting lessons from language design are relevant to RNN design. Language researchers stress the benefits of constraints: eliminating \texttt{GOTO} \citep{dijkstra:68}; introducing type-systems that prescribe the interfaces between parts of computer programs and guarantee their consistency \citep{pierce:02}; and working with stateless (pure) functions. 

For our purposes, types correspond to units as above. Let us therefore discuss the role of states. The reason for recurrent connections is precisely to introduce state-dependence. Unfortunately, state-dependent functions have side-effects -- unintended knock-on effects such as exploding gradients.

State-dependence without side-effects is not possible. The architectures proposed below encapsulate states in \emph{firmware} (which has no learned parameters) so that the \emph{learnware} (which encapsulates the parameters) is stateless. It follows that the learned features and gradients in strongly-typed architectures are better behaved and more interpretable than their classical counterparts, see section~\ref{sec:trnns}. 

Strictly speaking, the ideas from physics (to do with units) and functional programming (to do with states) are independent. However, we found that they complemented each other. We refer to architectures as strongly-typed when they both (i) preserve the type structure of their features and (ii) separate learned parameters from state-dependence. 

\paragraph{Overview.}
The core of the paper is section~\ref{sec:types}, which introduces strongly-typed linear algebra. As partial motivation, we show how types are implicit in principal component analysis and feedforward networks. A careful analysis of the update equations in vanilla RNNs identifies a flaw in classical RNN designs that leads to incoherent features. Fixing the problem requires new update equations that preserve the type-structure of the features.

Section~\ref{sec:trnns} presents strongly-typed analogs of standard RNN architectures. 
It turns out that small tweaks to the standard update rules yield simpler features and gradients, theorem~\ref{thm:sem} and corollary~\ref{thm:grad}. Finally, theorem~\ref{thm:rep} shows that, despite their more constrained architecture, strongly-typed RNNs have similar representational power to classical RNNs. Experiments in section~\ref{sec:exp} show that strongly-typed RNNs have comparable generalization performance and, surprisingly, lower training error than classical architectures (suggesting greater representational power). The flipside is that regularization appears to be more important for strongly-typed architectures, see experiments. 

\paragraph{Related work.}

The analogy between neural networks and functional programming was proposed in \citep{colah:15}, which also argued that representations should be interpreted as types. This paper extends Olah's proposal.
Prior work on typed-linear algebra \citep{macedo:13} is neither intended for nor suited to applications in machine learning. Many familiar RNN architectures already incorporate forms of \emph{weak-typing}, see section~\ref{sec:weak}.

\section{Strongly-Typed Features}
\label{sec:types}

A variety of type systems have been developed for mathematical logic and language design \citep{reynolds:74,girard:89,pierce:02}. We introduce a type-system based on linear algebra that is suited to deep learning. Informally, a \emph{type} is a vector space with an orthogonal basis. A more precise definition along with rules for manipulating types is provided below. Section~\ref{sec:egs} provides examples; section~\ref{sec:inconsistent} uses types to identify a design flaw in classical RNNs.

\subsection{Strongly-Typed Quasi-Linear Algebra}
\label{sec:stla}

Quasi-linear algebra is linear algebra supplemented with nonlinear functions that act coordinatewise.

\begin{defn}\label{def:type}
	Dot-products are denoted by $\langle\wt,\x\rangle$ or $\wt^\intercal\x$. 
	A \textbf{type} $\cT=\left(V, \langle\bullet,\bullet\rangle,\{\bt_i\}_{i=1}^d\right)$  is a $d$-dimensional vector space equipped with an inner product and an orthogonal basis such that $\langle\bt_i,\bt_j\rangle=\bO_{[i=j]}$.
\end{defn}
Given type $\cT$, we can represent vectors in $\vt\in V$ as real-valued $d$-tuples via
\begin{equation}
	\vt_\cT \leftrightarrow(v_1,\ldots, v_d)\in\bR^d
	\quad\text{where }v_i:= \langle\vt,\bt_i\rangle.
\end{equation}
\begin{defn}\label{def:admissible}
	The following operations are \emph{admissible}:
	\begin{enumerate}[T1.]
		\item \textbf{Unary operations on a type: $\cT\rightarrow \cT$}\\
		Given a function $f:\bR\rightarrow\bR$ (e.g. scalar multiplication, sigmoid $\sigma$, tanh $\tau$ or relu $\rho$), define
		\begin{equation}
			f(\vt) := \big(f(v_1),\ldots, f(v_d)\big)\in\cT.
		\end{equation}		
		\item \textbf{Binary operations on a type: $\cT\times \cT\rightarrow \cT$}\\
		Given $\vt,\wt \in\cT$ and an elementary binary operation $\mathsf{bin}\in\{+,-,\max,\min, \pi_1,\pi_2\}$\footnote{Note: $\pi_i$ is projection onto the $i^\text{th}$ coordinate.}, define
		\begin{equation}
			\mathsf{bin}(\vt,\wt) 
			:= \big(\mathsf{bin}(v_1,w_1),\ldots, \mathsf{bin}(v_d,w_d)\big).
		\end{equation}
		Binary operations on two different types (e.g. adding vectors expressed in different orthogonal bases) are \emph{not} admissible. 
		\item \textbf{Transformations between types: $\cT_1\rightarrow \cT_2$}\\
		A type-transform is a linear map $\Pt:V_1\rightarrow V_2$ such that $\Pt(\bt^{(1)}_i) = \bt^{(2)}_i$ for $i=\{1,\ldots,\min(d_1,d_2)\}$.
		Type-transformations are orthogonal matrices. 
		\item \textbf{Diagonalization: $\cT_1\rightarrow (\cT_2\rightarrow \cT_2)$}\\
		Suppose that $\vt\in \cT_1$ and $\wt\in \cT_2$ have the same dimension. Define
		\begin{equation}
			\vt_{\cT_1}\odot\wt_{\cT_2} := (v_1\cdot w_1,\ldots, v_d\cdot w_d)\in \cT_2,
		\end{equation}
		where $v_i:=\langle\vt,\bt^{(1)}_i\rangle$ and $w_i:=\langle\wt,\bt^{(2)}_i\rangle$. Diagonalization converts type $\cT_1$ into a new type, $\cT_2\rightarrow\cT_2$, that acts on $\cT_2$ by coordinatewise scalar multiplication.
	\end{enumerate}
\end{defn}
Definition~\ref{def:type} is inspired by how physicists have carved the world into an orthogonal basis of meters, amps, volts etc. The analogy is not perfect: e.g. $f(x)=x^2$ maps meters to square-meters, whereas types are invariant to coordinatewise operations. Types are looser than physical units.

\subsection{Motivating examples}
\label{sec:egs}

We build intuition by recasting PCA and feedforward neural nets from a type perspective.

\paragraph{Principal component analysis (PCA).}
Let $\X\in \bR^{n\times d}$ denote $n$ datapoints $\{\x^{(1)},\ldots,\x^{(n)}\}\subset \bR^d$. PCA factorizes 
$\X^\intercal\X = \Pt^\intercal\bD\Pt$
where $\Pt$ is a $(d\times d)$-orthogonal matrix and $\bD=\diag(\bd)$ contains the eigenvalues of $\X^\intercal \X$.
A common application of PCA is dimensionality reduction. From a type perspective, this consists in:
\begin{equation}
	\cT_{\{\be_k\}}\xrightarrow[(i)]{\Pt}\cT_{\{\bp_k\}}\xrightarrow[(ii)]{\proj}\cT_{\{\bp_k\}}\xrightarrow[(iii)]{\Pt^\intercal}\cT_{\{\be_k\}},
\end{equation}
(i) transforming the standard orthogonal basis $\{\be_k\}_{k=1}^d$ of $\bR^d$ into the \emph{latent type} given by the rows of $\Pt$;
(ii) projecting onto a subtype (subset of coordinates in the latent type); and 
(iii) applying the inverse to recover the original type.

\paragraph{Feedforward nets.}
The basic feedforward architecture is stacked layers computing $\bh = f(\Wt\cdot \x)$ where $f(\bullet)$ is a nonlinearity applied coordinatewise. We present two descriptions of the computation. 

The standard description is in terms of dot-products. Rows of $\Wt$ correspond to features, and matrix multiplication is a collection of dot-products that measure the similarity between the input $\x$ and the row-features:
\begin{equation}
	\Wt\x = \left(\begin{matrix}
		\cdots & \wt_1 & \cdots \\
		 &\vdots & \\
		\cdots & \wt_d & \cdots
	\end{matrix}\right)
	\x
	= \left(\begin{matrix}
		\langle \wt_{1}, \x\rangle \\
		\vdots \\
		 \langle\wt_{d}, \x\rangle
	\end{matrix}\right).
\end{equation}
Types provide a finer-grained description. Factorize $\Wt = \Pt \bD \Qt^\intercal$ by \emph{singular value decomposition} into $\bD=\diag(\bd)$ and orthogonal matrices $\Pt$ and $\Qt$. The layer-computation can be rewritten as $\bh = f(\Pt\bD \Qt^\intercal\x)$. From a type-perspective, the layer thus:
\begin{equation}
	\cT_\x \xrightarrow[(i)]{\Qt^\intercal} \cT_\text{latent}
	\xrightarrow[(ii)]{\diag(\bd)\odot\bullet}\cT_\text{latent}
	\xrightarrow[(iii)]{\Pt}\cT_\bh
	\xrightarrow[(iv)]{f(\bullet)}\cT_{\bh},
\end{equation}
(i) transforms $\x$ to a latent type; (ii) applies coordinatewise scalar multiplication to the latent type; (iii) transforms the result to the output type; and (iv) applies a coordinatewise nonlinearity. Feedforward nets learn interleaved sequences of type transforms and unary, type-preserving operations.

\subsection{Incoherent features in classical RNNs}
\label{sec:inconsistent}

There is a subtle inconsistency in classical RNN designs that leads to incoherent features. Consider the updates:
\begin{equation}
	\text{vanilla RNN:} \quad \bh_t = \sigma(\Vt\cdot\bh_{t-1} + \Wt\cdot\x_t + \bbb).
	\label{eq:rnn}
\end{equation}
We drop the nonlinearity, since the inconsistency is already visible in the linear case. Letting $\bz_t:=\Wt\x_t$ and unfolding Eq.~\eqref{eq:rnn} over time obtains
\begin{equation}
	\label{eq:rnn_unroll}
	\bh_t = \sum_{s=1}^t \Vt^{t-s}\cdot\bz_s.
\end{equation}
The inconsistency can be seen via dot-products and via types. From the dot-product perspective, observe that multiplying an input by a matrix squared yields
\begin{align}
	\Vt^2\bz 
	= \left(\begin{matrix}
		\vt_1  \\
		 \vdots  \\
		\vt_d 
	\end{matrix}\right)	
	\left(\begin{matrix}
	 \vdots & & \vdots \\
		\bc_1  &
		 \cdots &
		\bc_d \\
		\vdots & & \vdots 
	\end{matrix}\right)
	\bz
	= \left(\begin{matrix}
		\Big\langle(\vt_{1}^\intercal \bc_i\big)_{i=1}^d, \bz\Big\rangle \\
		\vdots \\
		\Big\langle(\vt_{d}^\intercal \bc_i\big)_{i=1}^d, \bz\Big\rangle 
	\end{matrix}\right),
\end{align}
\vspace{-10mm}

where $\vt_i$ refers to rows of $\Vt$ and $\bc_i$ to columns. Each coordinate of $\Vt^2\bz$ is computed by measuring the similarity of a row of $\Vt$ to all of its columns, and then measuring the similarity of the result to $\bz$. In short, features are tangled and uninterpretable.

From a type perspective, apply an SVD to $\Vt=\Pt\bD\Qt^\intercal$ and observe that $\Vt^2 = \Pt \bD \Qt^\intercal \Pt \bD \Qt^\intercal$. Each multiplication by $\Pt$ or $\Qt^\intercal$ transforms the input to a new type, obtaining
\begin{equation}
	\underbrace{\cT_\bh \xrightarrow{\bD\Qt^\intercal} \cT_\text{lat$_1$}
		\xrightarrow{\Pt}\cT_\text{lat$_2$}}_{\Vt}
	\underbrace{\xrightarrow{\bD\Qt^\intercal}\cT_\text{lat$_3$}
	\xrightarrow{\Pt}\cT_\text{lat$_4$}}_{\Vt}.
\end{equation}
Thus $\Vt$ sends $\bz\mapsto \cT_\text{lat$_2$}$ whereas $\Vt^2$ sends $\bz\mapsto \cT_\text{lat$_4$}$. Adding terms involving $\Vt$ and $\Vt^2$, as in Eq.~\eqref{eq:rnn_unroll}, entails adding vectors expressed in different orthogonal bases -- which is analogous to adding joules to volts. The same problem applies to LSTMs and GRUs. 

Two recent papers provide empirical evidence that recurrent (horizontal) connections are problematic even after gradients are stabilized: \citep{zaremba:15} find that Dropout performs better when restricted to vertical connections and \citep{laurent:15} find that Batch Normalization fails unless restricted to vertical connections \citep{ioffe:15}. More precisely, \citep{laurent:15} find that Batch Normalization improves training but not test error when restricted to vertical connections; it fails completely when also applied to horizontal connections. 

Code using \texttt{GOTO} can be perfectly correct, and RNNs with type mismatches can achieve outstanding performance. Nevertheless, both lead to spaghetti-like information/gradient flows that are hard to reason about. 

\paragraph{Type-preserving transforms.}
One way to resolve the type inconsistency, which we do not pursue in this paper, is to use symmetric weight matrices so that $\Vt = \Pt\bD\Pt^\intercal$ where $\Pt$ is orthogonal and $\bD=\diag(\bd)$. From the dot-product perspective,
\begin{equation}
	\bh_t = \sum_{s=1}^t \Pt\bD^{t-s}\Pt^\intercal\bz_s,
\end{equation}
which has the simple interpretation that $\bz$ is amplified (or dampened) by $\bD$ in the latent type provided by $\Pt$. From the type-perspective, multiplication by $\Vt^k$ is type-preserving
\begin{equation}
	\underbrace{\cT_\bh \xrightarrow{\Pt^\intercal}\cT_{\text{lat}_1}
		\xrightarrow{\bd^k\odot \bullet}\cT_{\text{lat}_1}\xrightarrow{\Pt}\cT_\bh}_{\Vt^k}
\end{equation}
so addition is always performed in the same basis. 

A familiar example of type-preserving transforms is autoencoders -- under the constraint that the decoder $\Wt^\intercal$ is the transpose of the encoder $\Wt$. Finally, \citep{moczulski:15} propose to accelerate matrix computations in feedforward nets by interleaving diagonal matrices, $\bA$ and $\bD$, with the orthogonal discrete cosine transform, $\bC$. The resulting transform, $\bA\bC\bD\bC^\intercal$, is type-preserving.

\section{Recurrent Neural Networks}
\label{sec:trnns}

We present three strongly-typed RNNs that purposefully mimic classical RNNs as closely as possible. Perhaps surprisingly, the tweaks introduced below have deep structural implications, yielding architectures that are significantly easier to reason about, see sections~\ref{sec:semantics} and \ref{sec:algebra}.

\subsection{Weakly-Typed RNNs}
\label{sec:weak}

We first pause to note that many classical architectures are \emph{weakly-typed}. That is, they introduce constraints or restrictions on off-diagonal operations on recurrent states.

The memory cell $\bc$ in LSTMs is only updated coordinate-wise and is therefore well-behaved type-theoretically -- although the overall architecture is not type consistent. The gating operation $\bz_t\odot\bh_{t-1}$ in GRUs \emph{reduces} type-inconsistencies by discouraging (i.e. zeroing out) unnecessary recurrent information flows.

SCRNs, or Structurally Constrained Recurrent Networks \citep{mikolov:15}, add a type-consistent state layer:
\begin{equation}
	\bs_t = \alpha\cdot \bs_{t-1} + (1-\alpha)\cdot \Wt_s \x_t,
	\quad\text{where }\alpha\text{ is a scalar.}
\end{equation}
In MUT1, the best performing architecture in \citep{jozefowicz:15}, the behavior of $\bz$ and $\bh$ is well-typed, although the gating by ${\mathbf r}$ is not. Finally, I-RNNs initialize their recurrent connections as the identity matrix \citep{le:15}. In other words, the key idea is a type-consistent initialization.

\subsection{Strongly-Typed RNNs}
\label{sec:rnns}

The vanilla strongly-typed RNN is
\begin{align}
	& \bz_t  = \Wt\x_t
	\label{eq:z}\\
	\text{T-RNN} \qquad & \ft_t  = \sigma(\Vt\x_t + \bbb)
	\label{eq:f}\\ 
	& \bh_t  = \ft_t\odot \bh_{t-1} + (1-\ft_t)\odot \bz_t
	\label{eq:h}		
\end{align}
The T-RNN has similar parameters to a vanilla RNN, Eq~\eqref{eq:rnn}, although their roles have changed. A nonlinearity for $\bz_t$ is not necessary because: (i) gradients do not explode, corollary~\ref{thm:grad}, so no squashing is needed; and (ii) coordinatewise multiplication by $\ft_t$ introduces a nonlinearity. Whereas relus are binary gates (0 if $\bz_t<0$, 1 else); the forget gate $\ft_t$ is a \emph{continuous} multiplicative gate on $\bz_t$.
 
Replacing the horizontal connection $\Vt\bh_{t-1}$ with a vertically controlled gate, Eq.~\eqref{eq:f}, stabilizes the type-structure across time steps. Line for line, the type structure is:
\begin{align}
	\begin{matrix}
	\cT_\x & \xrightarrow{\quad\eqref{eq:z}\quad} & \cT_\bh \\
	\cT_\x & \xrightarrow{\quad\eqref{eq:f}\quad} & \cT_\ft 
	 & \xrightarrow{\diag} & (\cT_{\bh}\rightarrow \cT_{\bh})\\
	(\underbrace{\cT_{\bh}\rightarrow \cT_{\bh}}_{\ft_t}) \times \underbrace{\cT_\bh}_{\bz_t}  & \xrightarrow[\bh_{t-1}]{\quad\eqref{eq:h}\quad} & \underbrace{\cT_\bh}_{\bh_t}
	\end{matrix}
\end{align}
We refer to lines~\eqref{eq:z} and \eqref{eq:f} as \emph{learnware} since they have parameters ($\Wt,\Vt,\bbb$). Line~\eqref{eq:h} is \emph{firmware} since it has no parameters. The firmware depends on the previous state $\bh_{t-1}$ unlike the learnware which is stateless. See section~\ref{sec:algebra} for more on learnware and firmware.

\paragraph{Strongly-typed LSTMs}
differ from LSTMs in two respects: (i)  $\x_{t-1}$ is substituted for $\bh_{t-1}$ in the first three equations so that the type structure is coherent; and (ii) the nonlinearities in $\bz_t$ and $\bh_t$ are removed as for the T-RNN.
\begin{align}
	& \bz_t = \tau(\Vt_{z}\bh_{t-1}+\Wt_{z}\x_t + \bbb_z)\\
	& \ft_t = \sigma(\Vt_{f}\bh_{t-1} + \Wt_{f}\x_t + \bbb_f)\\
	\text{LSTM}\qquad& \bo_t = \tau(\Vt_{o}\bh_{t-1}+\Wt_{o}\x_t + \bbb_o)\\
	& \bc_t =  \ft_t \odot \bc_{t-1} + (1-\ft_t) \odot \bz_t  \\
	& \bh_t = \tau(\bc_t) \odot \bo_t\\
	\\
	& \bz_t = \Vt_{z}\x_{t-1} + \Wt_{z}\x_{t}+\bbb_z\\
	& \ft_t = \sigma(\Vt_{f}\x_{t-1} + \Wt_{f}\x_{t} + \bbb_f) \\
	\text{T-LSTM}\qquad& \bo_t = \tau(\Vt_{o}\x_{t-1} + \Wt_{o}\x_{t}+\bbb_o)\\
	& \bc_t = \ft_t \odot \bc_{t-1} + (1-\ft_t) \odot \bz_t \\
	& \bh_t = \bc_t \odot \bo_t\\
\end{align}
We drop the input gate from the updates for simplicity; see \citep{greff:15}. The type structure is
\begin{align}
	\begin{matrix}
	\cT_\x & \xrightarrow{\quad} & \cT_\bc \\
	\cT_\x & \xrightarrow{\quad} & \cT_\ft 
	 & \xrightarrow{\diag} & (\cT_{\bc}\rightarrow \cT_{\bc})\\
	 \cT_\x & \xrightarrow{\quad} & \cT_\bh \\
	(\cT_{\bc}\rightarrow \cT_{\bc}) \times \cT_\bc  & \xrightarrow[\bc_{t-1}]{} & \cT_\bc
	& \xrightarrow{\diag} & (\cT_{\bh}\rightarrow \cT_{\bh})\\
	(\cT_{\bh}\rightarrow \cT_{\bh}) \times \cT_\bh & \xrightarrow{\quad}& \cT_\bh
	\end{matrix} 
\end{align}

\paragraph{Strongly-typed GRUs}
adapt GRUs similarly to how LSTMs were modified. In addition, the reset gate $\bz_t$ is repurposed; it is no longer needed for weak-typing.
\begin{align}
	& \bz_t = \sigma(\Vt_{z} \bh_{t-1} + \Wt_{z} \x_t + \bbb_z) \\
	\text{GRU}\qquad 
	& \ft_t = \sigma(\Vt_{f} \bh_{t-1} + \Wt_{f} \x_t + \bbb_f) \\ 
	& \bo_t = \tau\big(\Vt_{o} (\bz_t\odot \bh_{t-1}) + \Wt_{o} \x_t + \bbb_o\big) \\
	& \bh_t = \ft_t \odot \bh_{t-1} + (1-\ft_t) \odot \bo_t\\
	\\
	& \bz_t = \Vt_{z} \x_{t-1} + \Wt_{z} \x_{t} + \bbb_z \\ 
	\text{T-GRU} \qquad
	& \ft_t = \sigma(\Vt_{f} \x_{t-1} + \Wt_{f} \x_{t} + \bbb_f) \\
	& \bo_t = \tau(\Vt_{o} \x_{t-1} + \Wt_{o} \x_{t} + \bbb_o) \\
	& \bh_t = \ft_t \odot \bh_{t-1} + \bz_t\odot \bo_t\\
\end{align}
The type structure is
\begin{align}
	\begin{matrix}
	\cT_\x &\xrightarrow{\quad} &\cT_\bh & & \\
	\cT_\x &\xrightarrow{\quad}& \cT_\ft
	 & \xrightarrow{\diag} & (\cT_{\bh}  \rightarrow  \cT_{\bh})\\
	 \cT_\x & \xrightarrow{\quad} & \cT_\bo 
	 & \xrightarrow{\diag} & (\cT_{\bh}  \rightarrow  \cT_{\bh})\\
	(\cT_{\bh}\rightarrow \cT_{\bh}) & \times & (\cT_{\bh} \rightarrow  \cT_{\bh}) \times \cT_\bh  &\xrightarrow[\bh_{t-1}]{} &\cT_\bh	
	\end{matrix}
\end{align}

\subsection{Feature Semantics}
\label{sec:semantics}

The output of a vanilla RNN expands as the uninterpretable
\begin{equation}
	\bh_t = \sigma(\Vt\sigma(\Vt\sigma(\cdots)  + \Wt\x_{t-1} + \bbb) + \Wt\x_t + \bbb),
\end{equation}
with even less interpretable gradient. Similar considerations hold for LSTMs and GRUs. Fortunately, the situation is more amenable for strongly-typed architectures. In fact, their semantics are related to \emph{average-pooled convolutions}.

\paragraph{Convolutions.}
Applying a one-dimensional convolution to input sequence $\x[t]$ yields output sequence
\begin{equation}
	\bz[t] = (\Wt * \x)[t] = \sum_{s} \Wt[s]\cdot \x[t-s]
\end{equation}
Given weights $f_s$ associated with $\Wt[s]$, average-pooling yields $\bh_t = \sum_{s=1}^t f_s\cdot \bz[s]$. A special case is when the convolution applies the same matrix to every input:
\begin{equation}
	\Wt[s] = \begin{cases}
		\Wt & \text{if }s=0\\
		0 & \text{else.}
	\end{cases}
\end{equation}
The average-pooled convolution is then a weighted average of the features extracted from the input sequence. 

\paragraph{Dynamic temporal convolutions.}
We now show that strongly-typed RNNs are one-dimensional temporal convolutions with \emph{dynamic} average-pooling. Informally, strongly-typed RNNs transform input sequences into a weighted average of features extracted from the sequence
\begin{equation}
	\x_{1:t} \mapsto \expec_{\bP_{\x_{1:t}}}\Big[\Wt*\x\Big] = \sum_{s=1}^t \bP_{\x_{1:t}}(s)\cdot (\Wt\cdot \x_s) =: \bh[t]
\end{equation}
where the weights depends on the sequence. In detail:
\begin{thm}[feature semantics via dynamic convolutions]\label{thm:sem}
	Strongly-typed features are computed explicitly as follows.
	\begin{itemize}
		\item \emph{T-RNN.} The output is $\bh_t = \expec_{s\sim\bP_{\x_{1:t}}}\big[\Wt\x_s\big]$ where
		\begin{equation}
			\bP_{\x_{1:t}}(s) = \begin{cases}
				1-\ft_t & \text{if }s=t\\
				\ft_t\odot \bP_{\x_{1:t-1}}(s) & \text{else.}
			\end{cases}
		\end{equation}
		\item \emph{T-LSTM.}
		Let $\bU_{\bullet}:=[\Vt_{\bullet};\Wt_{\bullet};\bbb_\bullet]$ and $\tilde{\x}_t:=[\x_{t-1};\x_t;1]$ denote the vertical concatenation of the weight matrices and input vectors respectively. Then,
		\begin{equation}
		\bh_t = \tau\big(\bU_o\tilde{\x}_t\big)\odot \expec_{s\sim\bP_{\x_{1:t}}}\big[\bU_z\tilde{\x}_s\big]
		\end{equation}
		where $\bP_{\x_{1:t}}$ is defined as above.
		\item \emph{T-GRU.}
		Using the notation above, 
		\begin{equation}
		\bh_t = \sum_{s=1}^t \Ft_s\odot
		\Big(\tau\big(\bU_o\tilde{\x}_s\big)\odot \bU_z\tilde{\x}_s\Big)
		\end{equation}
		where 
		\begin{equation}
			\Ft_s = \begin{cases}
				1 & \text{if }s=t\\
				\ft_s\odot \Ft_{s+1} & \text{else.}
			\end{cases}
		\end{equation}
		\end{itemize}
\end{thm}

\begin{proof}
	Direct computation.
\end{proof}

In summary, T-RNNs compute a dynamic distribution over time steps, and then compute the expected feedforward features over that distribution. T-LSTMs store expectations in private memory cells that are reweighted by the output gate when publicly broadcast. Finally, T-GRUs drop the requirement that the average is an expectation, and also incorporate the output gate into the memory updates.

Strongly-typed gradients are straightforward to compute and interpret:

\begin{cor}[gradient semantics]\label{thm:grad}
	The strongly-typed gradients are
	\begin{itemize}
		\item T-RNN:
		\begin{align}
			\frac{\dd \bh_t}{\dd \Wt} & = \expec_{s\sim \bP_{\x_{1:t}}}\Big[ \frac{\dd}{\dd \Wt}(\bz_s)\Big]
			\\
			\frac{\dd \bh_t}{\dd \Vt} & = \expec_{s\sim \bP_{\x_{1:t}}}\Big[\bz_s\odot \frac{\dd}{\dd \Vt}\big(\log\bP_{\x_{1:t}}(s)\big)\Big]\qquad\quad
		\end{align}
		and similarly for $\frac{\dd}{\dd \bbb}$.
		\item T-LSTM:
		\begin{align}
			\frac{\dd \bh_t}{\dd \bU_o} & = 
			\frac{\dd }{\dd \bU_o}(\bo_t) \odot \expec_{s\sim\bP_{\x_{1:t}}}\Big[\bz_s\Big] \\
			\frac{\dd \bh_t}{\dd \bU_z} & = \bo_t\odot \expec_{s\sim\bP_{\x_{1:t}}}\Big[\frac{\dd}{\dd \bU_z}(\bz_s)\Big] \\ 
			\frac{\dd \bh_t}{\dd \bU_f} & = \bo_t\odot \expec_{s\sim\bP_{\x_{1:t}}}\Big[\bz_s\odot \frac{\dd}{\dd \bU_f}\big(\log \bP_{\x_{1:t}}(s)\big)\Big] 		
		\end{align}
		\item T-GRU:
		\begin{align}
			\frac{\dd \bh_t}{\dd \bU_o} & = \sum_{s=1}^t \Ft_s\odot
			\frac{\dd}{\dd \bU_o} (\bo_s)\odot \bz_s\\
			\frac{\dd \bh_t}{\dd \bU_z} & = \sum_{s=1}^t \Ft_s\odot
			\bo_s\odot \frac{\dd}{\dd \bU_z}(\bz_s)\qquad\qquad\qquad\\
			\frac{\dd \bh_t}{\dd \bU_f} & = \sum_{s=1}^t \frac{\dd }{\dd \bU_f}\big(\Ft_s\big)\odot
			\bo_s\odot \bz_s
		\end{align}
	\end{itemize}
\end{cor}
It follows immediately that gradients will not explode for T-RNNs or LSTMs. Empirically we find they also behave well for T-GRUs.

\subsection{Feature Algebra}
\label{sec:algebra}

A vanilla RNN can approximate any continuous state update $\bh_{t} = g(\x_t,\bh_{t-1})$ since $\text{span}\{s(\wt^\intercal\x)\,|\,\wt\in \bR^d\}$ is dense in continuous functions ${\mathcal C}(\bR^d)$ on $\bR^d$ if $s$ is a nonpolynomial nonlinear function \citep{leshno:93}. It follows that vanilla RNNs can approximate any recursively computable partial function \citep{siegelmann:95}. 

Strongly-typed RNNs are more constrained. We show the constraints reflect a coherent design-philosophy and are less severe than appears.

\paragraph{The learnware / firmware distinction.}
Strongly-typed architectures factorize into stateless \emph{learnware} and state-dependent \emph{firmware}. For example, T-LSTMs and T-GRUs factorize\footnote{A superficially similar factorization holds for GRUs and LSTMs. However, their learnware is \emph{state-dependent}, since $(\ft_t,\bz_t,\bo_t)$ depend on $\bh_{t-1}$.} as
\begin{align}
	(\ft_t,\bz_t,\bo_t) & = \text{T-LSTM}^\text{learn}_{\{\Vt_\bullet, \Wt_\bullet, \bbb_\bullet\}}(\x_{t-1},\x_t) \\
	(\bh_t,\bc_t) & = \text{T-LSTM}^\text{firm}(\ft_t,\bz_t,\bo_t;\underbrace{\bc_{t-1}}_{\text{state}})
	\\
	(\ft_t,\bz_t,\bo_t) & = \text{T-GRU}^\text{learn}_{\{\Vt_\bullet, \Wt_\bullet, \bbb_\bullet\}}(\x_{t-1},\x_t) \\
	\bh_{t} & = \text{T-GRU}^\text{firm}(\ft_t,\bz_t,\bo_t; \underbrace{\bh_{t-1}}_{\text{state}}).
\end{align}
Firmware decomposes coordinatewise, which prevents side-effects from interacting: e.g. for T-GRUs
\begin{gather}
	\text{T-GRU}^\text{firm}(\ft,\bz,\bo; \bh) = 
	\Big(\varphi(f^{(i)},z^{(i)},o^{(i)};h^{(i)})
	\Big)_{i=1}^d\\
	\text{where }
	\varphi(f,z,o;h) = fh + zo
\end{gather}
and similarly for T-LSTMs. Learnware is stateless; it has no side-effects and does not decompose coordinatewise. Evidence that side-effects are a problem for LSTMs can be found in \citep{zaremba:15} and \citep{laurent:15}, which show that Dropout and Batch Normalization respectively need to be restricted to vertical connections.

In short, under strong-typing the learnware carves out features which the firmware uses to perform \emph{coordinatewise} state updates $h^i_t = g(h_{t-1}^i, z_{t-1}^i)$. Vanilla RNNs allow \emph{arbitrary} state updates $\bh_t = g(\bh_{t-1},\x_{t})$. LSTMs and GRUs restrict state updates, but allow arbitrary functions of the state. Translated from a continuous to discrete setting, the distinction between strongly-typed and classical architectures is analogous to working with binary logic gates (AND, OR) on variables $\bz_t$ learned by the vertical connections -- versus working directly with $n$-ary boolean operations.

\paragraph{Representational power.}
Motivated by the above, we show that a minimal strongly-typed architecture can span the space of continuous \emph{binary} functions on features.
\begin{thm}[approximating binary functions]\label{thm:rep}
	The strongly-typed \emph{minimal} RNN with updates
	\begin{equation}
		\text{T-MR:}\quad \bh_t = \rho(\bbb\odot \bh_{t-1} + \Wt\x_t + \bc)
	\end{equation}
	and parameters $(\bbb$, $\bc$, $\Wt)$ can approximate any set of continuous binary functions on features.
\end{thm}
\begin{proofsketch}
	Let $z = \wt^\intercal \x$ be a feature of interest.
	Combining \citep{leshno:93} with the observation that $a\rho(bh + z + c)=\rho(abh+az+ac)$ for $a>0$ implies that $\overline{\text{span}}\{\rho(b \cdot h_{t-1} + z_t)\,|\,b,c\in\bR\}={\mathcal C}(\bR^2)$. As many weighted copies $a z$ of $z$ as necessary are obtained by adding rows to $\Wt$ that are scalar multiples of $\wt$.

	Any set of binary functions on any collection of features can thus be approximated. Finally, vertical connections can approximate any set of features \citep{leshno:93}.
\end{proofsketch}

\section{Experiments}
\label{sec:exp}
\begin{table*}[!htb]
\caption{The \textbf{(train, test) cross-entropy loss} of RNNs and T-RNNs on WP dataset.}
\label{tab:wp_rnnres}
\centering
\begin{tabular}{| l | c | c | c | c | c | c |}
\hline
Model & \multicolumn{3}{|c|}{vanilla RNN} & \multicolumn{3}{|c|}{T-RNN} \\
\hline
\hline
Layers & 1 & 2 & 3 & 1 & 2 & 3 \\
\hline
64 (no dropout) & (1.365, 1.435)  &  (1.347, 1.417) & (1.353, 1.423)  &   (1.371, 1.452) & (\textbf{1.323}, 1.409)  &  (1.342, 1.423)\\
256 & (1.215, 1.274)  &  (1.242, \textbf{1.254}) &  (1.257, 1.273)& (1.300, 1.398)&  (1.251, 1.276) & (1.233, 1.266) \\
\hline
\end{tabular}
\end{table*}

\begin{table*}[!htb]
\caption{The \textbf{(train, test) cross-entropy loss} of LSTMs and T-LSTMs on WP dataset.}
\label{tab:wp_lstmres}
\centering
\begin{tabular}{| l | c | c | c | c | c | c |}
\hline
Model & \multicolumn{3}{|c|}{LSTM} & \multicolumn{3}{|c|}{T-LSTM} \\
\hline
\hline
Layers & 1 & 2 & 3 & 1 & 2 & 3 \\
\hline
64 (no dropout) & (1.496, 1.560) & (1.485, 1.557) & (1.500, 1.563) &(1.462, 1.511) & (\textbf{1.367}, 1.432) & (1.369, 1.434)  \\
256 & (1.237, 1.251) & (1.098, 1.193) & (1.185, 1.213) & (1.254, 1.273) &  (1.045, \textbf{1.189}) & (1.167, 1.198)\\
\hline
\end{tabular}
\end{table*}

\begin{table*}[!htb]
\caption{The \textbf{(train, test) cross-entropy loss} of GRUs and T-GRUs on WP dataset.}
\label{tab:wp_grures}
\centering
\begin{tabular}{| l | c | c | c | c | c | c |}
\hline
Model & \multicolumn{3}{|c|}{GRU} & \multicolumn{3}{|c|}{T-GRU} \\
\hline
\hline
Layers & 1 & 2 & 3 & 1 & 2 & 3 \\
\hline
64 (no dropout)& (1.349, 1.435) & (1.432, 1.503) & (1.445, 1.559) &  (1.518 ,1.569) & (\textbf{1.337}, 1.422) & (1.377, 1.436)\\
256 & (1.083, 1.226) & (1.163, 1.214) & (1.219, 1.227) & (1.142, 1.296) & (1.208, 1.240) & (1.216, \textbf{1.212}) \\
\hline
\end{tabular}
\end{table*}
We investigated the empirical performance of strongly-typed recurrent nets for sequence learning. 
The performance was evaluated on character-level and word-level text generation.
We conducted a set of proof-of-concept experiments. The goal is not to compete with previous work or to find the best performing model under a specific hyper-parameter setting. Rather, we investigate how the two classes of architectures perform over a range of settings.

\subsection{Character-level Text Generation}
The first task is to generate text from a sequence of characters by predicting the next character in a sequence.
We used Leo Tolstoy's \emph{War and Peace} (WP) which consists of 3,258,246 characters of English text, split into train/val/test sets with 80/10/10 ratios.
The characters are encoded into $K$-dimensional \emph{one-hot} vectors, where $K$ is the size of the vocabulary.
We follow the experimental setting proposed in \citep{karpathy:15a}. Results are reported for two configurations: ``64'' and ``256'', 
which are models with the same number of parameters as a 1-layer LSTM with 64 and 256 cells per layer respectively.
Dropout regularization was only applied to the ``256'' models.  The dropout rate was taken from $\{0.1, 0.2 \}$ based on validation performance.
Tables \ref{tab:wp_lstmres} and \ref{tab:wp_grures} summarize the performance in terms of cross-entropy loss $H(\y, \bp) = \sum_{i=1}^K y_i\log p_i$.

We observe that the training error of strongly-typed models is typically lower than that of the standard models for $\geq 2$ layers. The test error of the two architectures are comparable. However, our results (for both classical and typed models) fail to match those reported in \citep{karpathy:15a}, where a more extensive parameter search was performed.

\begin{table}[!htb]
 \caption{Perplexity on the Penn Treebank dataset. }
\label{tab:ptb_res}
 \centering
 \begin{tabular}{| l || c | c | c | c | c | }
  \hline
  Model & Train & Validation & Test  \\
  \hline
  \multicolumn{4}{|c|}{small, no dropout} \\
  \hline
  vanilla RNN & 416.50 & 442.31 & 432.01  \\
  T-RNN & 58.66 & 172.47 & 169.33  \\

  \hline
  LSTM  & 36.72 & 122.47 & \textbf{117.25}  \\
  T-LSTM & \textbf{28.15} & 215.71 & 200.39 \\
  \hline
  GRU  & 31.14 & 179.47 & 173. 27  \\
T-GRU & 28.57 & 207.94 & 195.82 \\
  \hline 
  \multicolumn{4}{|c|}{medium, with dropout} \\
  \hline
  LSTM \tiny{\citep{zaremba:15}} & \textbf{48.45} & 86.16 & 82.70 \\
  LSTM (3-layer) & 71.76 & 98.22 & 97.87 \\ 
  T-LSTM & 50.21 & 87.36 & 82.71 \\
  T-LSTM (3-layer) & 51.45 & 85.98 & \textbf{81.52} \\
  \hline
  GRU & 65.80 & 97.24 & 93.44 \\
 T-GRU & 55.31 & 121.39 & 113.85 \\
  \hline
 \end{tabular}
\end{table}

\subsection{Word-level Text Generation}
The second task was to generate word-level text by predicting the next word from a sequence.
We used the Penn Treebank (PTB) dataset \citep{marcus:1993}, which consists of 929K training words, 73K validation words, and 82K test words, with vocabulary size of 10K words.
The PTB dataset is publicly available on web.\footnote{\url{http://www.fit.vutbr.cz/~imikolov/rnnlm/simple-examples.tgz}}

We followed the experimental setting in \citep{zaremba:15} and compared the performance of ``small'' and ``medium'' models. 
The parameter size of ``small'' models is equivalent to that of $2$ layers of $200$-cell LSTMs, while the parameter size of ``medium'' models is the same as that of $2$ layers of $650$-cell LSTMs.
For the ``medium'' models, we selected the dropout rate from \{0.4, 0.5, 0.6\} according to validation performance.
Single run performance, measured via \emph{perplexity}, i.e., $\exp(H(\y, \bp))$, are reported in Table \ref{tab:ptb_res}.

\paragraph{Perplexity.}
For the ``small'' models, we found that the training perplexity of strongly-typed models is consistently lower than their classical counterparts, in line with the result for War \& Peace. Test error was significantly worse for the strongly-typed architectures. A possible explanation for both observations is that strongly-typed architectures require more extensive regularization. 

An intriguing result is that the T-RNN performs in the same ballpark as LSTMs, with perplexity within a factor of two. By contrast, the vanilla RNN fails to achieve competitive performance. This suggests there may be strongly-typed architectures of intermediate complexity between RNNs and LSTMs with comparable performance to LSTMs.

The dropout-regularized ``medium'' T-LSTM matches the LSTM performance reported in \citep{zaremba:15}. The 3-layer T-LSTM obtains slightly better performance. The results were obtained with almost identical parameters to Zaremba et al (the learning rate decay was altered), suggesting that T-LSTMs are viable alternatives to LSTMs for sequence learning tasks when properly regularized.
Strongly-typed GRUs did not match the performance of GRUs, possibly due to insufficient regularization.

\paragraph{Gradients.}
We investigated the effect of removing gradient clipping on medium-sized LSTM and T-LSTM. T-LSTM gradients are well-behaved without clipping, although test performance is not competitive. In contrast, LSTM gradients explode without clipping and the architecture is unusable. It is possible that carefully initialized T-LSTMs may be competitive without clipping. We defer the question to future work.

\paragraph{Runtime.}
Since strongly-typed RNNs have fewer nonlinearities than standard RNNs, 
we expect that they should have lower computational complexity.
Training on the PTB dataset on an NVIDIA GTX 980 GPU, we found that T-LSTM is on average $\sim1.6\times$ faster than LSTM. Similarly, the T-GRU trains on average $\sim1.4\times$ faster than GRU.

\section{Conclusions}

RNNs are increasingly important tools for speech recognition, natural language processing and other sequential learning problems. The complicated structure of LSTMs and GRUs has led to searches for simpler alternatives with limited success \citep{bayer:09,greff:15,jozefowicz:15,le:15,mikolov:15}. This paper introduces strong-typing as a tool to guide the search for alternate architectures. In particular, we suggest searching for update equations that learn well-behaved features, rather than update equations that ``appear simple''. We draw on two disparate intuitions that turn out to be surprisingly compatible: (i) that neural networks are analogous to measuring devices \citep{balduzzi:11idm} and (ii) that training an RNN is analogous to writing code.

The main contribution is a new definition of type that is closely related to singular value decomposition -- and is thus well-suited to deep learning. It turns out that classical RNNs are badly behaved from a type-perspective, which motivates modifying the architectures. Section~\ref{sec:trnns} tweaked LSTMs and GRUs to make them well-behaved from a typing and functional programming perspective, yielding features and gradients that are easier to reason about than classical architectures. 

Strong-typing has implications for the depth of RNNs. It was pointed out in \citep{pascanu:14a} that unfolding horizontal connections over time implies the concept of depth is not straightforward in classical RNNs. By contrast, depth has the same meaning in strongly-typed architectures as in feedforward nets, since vertical connections learn features and horizontal connections act coordinatewise.

Experiments in section~\ref{sec:exp} show that strongly-typed RNNs achieve comparable generalization performance to classical architectures when regularized with dropout and have consistently lower training error. It is important to emphasize that the experiments are not conclusive. Firstly, we did not deviate far from settings optimized for classical RNNs when training strongly-typed RNNs. Secondly, the architectures were chosen to be as close as possible to classical RNNs. A more thorough exploration of the space of strongly-typed nets may yield better results. 

\paragraph{Towards machine reasoning.}
A definition of machine reasoning, adapted from \citep{bottou:14}, is ``algebraically manipulating features to answer a question''. Hard-won experience in physics \citep{chang:04}, software engineering \citep{dijkstra:68}, and other fields has led to the conclusion that well-chosen constraints are crucial to effective reasoning. Indeed, neural Turing machines \citep{graves:14} are harder to train than more constrained architectures such as neural queues and deques \citep{grefenstette:15}.

Strongly-typed features have a consistent semantics, theorem~\ref{thm:sem}, unlike features in classical RNNs which are rotated across time steps -- and are consequently difficult to reason about. We hypothesize that strong-typing will provide a solid foundation for algebraic operations on learned features. Strong-typing may then provide a useful organizing principle in future machine reasoning systems.

\paragraph{Acknowledgements.}
We thank Tony Butler-Yeoman, Marcus Frean, Theofanis Karaletsos, JP Lewis and Brian McWilliams for useful comments and discussions.

\end{document}